\documentclass[journal, twocolumn]{IEEEtran}

\usepackage[utf8]{inputenc} 
\usepackage[T1]{fontenc}    
\usepackage{hyperref}       
\usepackage{url}            
\usepackage{booktabs}       
\usepackage{amsfonts}       
\usepackage{nicefrac}       
\usepackage{microtype}      
\usepackage{amsmath,amssymb}
\usepackage{graphicx}
\usepackage{textcomp}
\usepackage{xcolor}
\usepackage{amsthm,amssymb,bm, verbatim,dsfont,mathtools}
\usepackage[mathcal]{euscript}
\usepackage{color,graphicx}
\usepackage{subfigure}
\usepackage{etoolbox}
\usepackage{tikz}
\usepackage{xr,xspace}
\usepackage{todonotes}
\usepackage{caption,soul}
\usepackage{accents}
\usepackage{comment}
\usepackage{cite}
\usepackage{csquotes}
\usepackage{algorithm}
\usepackage{algorithmic}

\newtheorem{theorem}{Theorem}
\newtheorem{lemma}{Lemma}
\newtheorem{remark}{Remark}
\newcommand{\expect}[1]{\mathbb{E}\left[ #1 \right]}
\newcommand{\identityf}[1]{\mathbf 1_{\{#1\}}}
\newlength{\dhatheight}

\begin{document}

\title{Decentralized Heterogeneous Multi-Player Multi-Armed Bandits with Non-Zero Rewards on Collisions}

\author{Akshayaa Magesh,~\IEEEmembership{Graduate Student Member,~IEEE,}
        and Venugopal V. Veeravalli,~\IEEEmembership{Fellow,~IEEE}
\thanks{A. Magesh and V.V. Veeravalli are with the Department of Electrical and Computer Engineering, University of Illinois at Urbana-Champaign, IL, 61820 USA. (email: amagesh2@illinois.edu, vvv@illinois.edu)}
\thanks{This research was supported by the US National Science Foundation SpecEES program, under grant number 1730882, through the University of Illinois at Urbana-Champaign.}}

%



\maketitle

\begin{abstract}
We consider a fully decentralized multi-player stochastic multi-armed bandit setting where the players cannot communicate with each other and can observe only their own actions and rewards. The environment may appear differently to different players, $\textit{i.e.}$, the reward distributions for a given arm are heterogeneous across players. In the case of a collision (when more than one player plays the same arm), we allow for the colliding players to receive non-zero rewards. The time-horizon $T$ for which the arms are played is \emph{not} known to the players. 
Within this setup, where the number of players is allowed to be greater than the number of arms,  we present a policy that achieves near order-optimal expected regret of order $O(\log^{1 + \delta} T)$ for $\delta >0$ (however small)  over a time-horizon of duration $T$.
\end{abstract}


%
\IEEEpeerreviewmaketitle

\section{Introduction}
%
%
%
%
\IEEEPARstart{T}{he} multi-armed bandit (MAB) is a well-studied model for sequential decision-making problems with an inherent exploration-exploitation trade-off. MABs have seen applications in recommendation systems, advertising, ranking results of search engines, and more. The classical stochastic MAB setup considers an agent/player, who at each time instant $t$, chooses an action from a finite set of actions (or arms). The agent receives a reward drawn from an unknown distribution associated with the arm chosen. The goal is to design a decision-making policy that maximizes the agent's average cumulative reward, or equivalently, minimizes the average accumulated regret. Policies that are designed to minimize regret in bandit settings aim to achieve sub-linear regret with respect to the time horizon $T$. The MAB problem was first considered in the context of clinical trials by Thompson \cite{thompson}, who introduced a posterior sampling heuristic commonly known as Thompson sampling. In their seminal work, Lai and Robbins \cite{lairob} formalized the stochastic MAB setting and provided a lower bound on the average regret of order $\Omega(\log T)$ for time horizon $T$. They also presented an asymptotically optimal decision policy using the idea of upper confidence bounds, which was further explored in \cite{auer} and \cite{garivier11a}. Other variants of the MAB setup such as adversarial, contextual and Markovian have also been studied in literature (see, e.g., \cite{bubeck,lattimore}).
\subsection{Multi-player multi-armed bandits}
Recently, there has been a growing interest in the study of \emph{multi-player} MAB settings, where instead of a single agent, there are $K$ agents simultaneously pulling the arms at each time instant. The average system regret is defined with respect to the optimal assignment of arms that maximizes the sum of expected rewards of all players (which can be interpreted as the system performance). 
The event of multiple players pulling the same arm simultaneously is commonly referred to as a \emph{collision}, and leads to the players receiving reduced or zero rewards. Thus, while designing policies for the multi-player MAB setting, in addition to balancing the exploration-exploitation trade-off, it is important to control the number of collisions that occur.

Multi-player bandit models can be broadly classified into centralized and decentralized settings. In the centralized setting, there exists a central controller that can coordinate the actions of all the players. In this case, the multi-player problem can be reduced to a single agent MAB problem, where the agent is the set of all players taken collectively, and this agent can choose multiple arms at a time as directed by the controller. However, the communication overhead placed by the central controller, and the communication bottleneck at the controller might be prohibitive. Therefore,  it is of interest to study a decentralized system without central control, which is the focus of our study. A tight (in the order sense) lower bound for the system regret for the centralized case is of course  the same as that for a single agent multi-armed bandit setup, \textit{i.e.},  $\Omega(\log T)$, which also serves as a lower bound on the system regret for the decentralized case. It should be noted that no larger lower bounds have been proven for the decentralized case. 

Multi-player MABs with \textit{cooperative agents} in the decentralized setting (\textit{i.e.}, where the players communicate among themselves in order to achieve a common objective) have seen applications in geographically distributed ad servers \cite{cesa16}, peer-to-peer networks \cite{szorenyi13}, and recommendation systems \cite{korda}. Multi-player stochastic MABs in the decentralized setting are particularly relevant to cognitive radio and dynamic spectrum access systems \cite{biglieri2013principles}.
In these systems, the finite number of channels representing different frequency bands are treated as arms, the users in the network are treated as players, and the data rates received from the channels can be interpreted as rewards. Since the maximum rate that can be received from a channel is limited, assuming bounded rewards for the arms is justified. In this setting, the players are competing for the same set of finite resources (the rewards received from the finite set of arms).  
\subsection{Previous related work on decentralized multi-player MABs}
Some of the prior work in the decentralized setting assumes that communication between the players is possible \cite{evirgen, avnercomm}. However, cooperation through communication between the players imposes an additional cost and may suffer from latency issues due to delays. Other works assume that sensing occurs at the level of every individual player, such as smart devices in a network being able to sense if a channel is being used or not without transmitting on it. The work in \cite{dileep} considers such a setting where an auction algorithm is used by players to come to a consensus on the optimal assignment of arms. However in other systems, such as emerging architectures in Internet of Things (IoT), the individual nodes may not be capable of such sensing. This motivates the study of a fully decentralized scenario, where there is no central control and the players cannot communicate with each other in any manner. The players can observe only their own actions and rewards.

In most of the prior work on the fully decentralized setting, the assumption is made that the reward distribution for any arm is the same (homogeneous) for all players. This setting was first considered in \cite{mega}, where prior knowledge of the number of players is not assumed. The algorithm presented in \cite{mega}, named Multi-user $\epsilon$-Greedy collision Avoiding algorithm (MEGA), combines a probabilistic $\epsilon$-greedy algorithm with a collision avoiding mechanism inspired by the ALOHA protocol, and provides guarantees of sub-linear regret.  An algorithm named Musical Chairs is proposed in \cite{mc}, which is composed of a learning phase for the players to learn an $\epsilon$-correct ranking of arms and the number of players,  and a  `Musical Chairs' phase, in which the $K$ players fix on the top $K$ arms. High probability guarantees of constant regret are provided in \cite{mc} for the Musical Chairs algorithm. The fully decentralized setting with homogeneous reward distributions across the players is also considered in \cite{anandkumar, perchet,besson}. All of the above mentioned works also assume that in the event of a collision, all the colliding players receive {\em zero} rewards. 

In cognitive radio and uncoordinated dynamic spectrum access networks, the users are usually not colocated physically, and therefore the reward distributions for a given arm may be heterogeneous across users. There have been a few works that study such a heterogeneous setting. In \cite{dileep, hanawal} the heterogeneous setting is studied under the assumption that the players are capable of sensing (\textit{i.e.}, players can observe whether an arm is being used or not without pulling it). A fully decentralized heterogeneous setting is studied in \cite{kaufmann, asilomar, got, bistritz2020game}, where players can observe only their own actions and rewards. In \cite{kaufmann, asilomar} it is assumed that in the event of a collision, the colliding players get zero rewards; the idea of forced collisions  is used to enable the players to communicate with one another and settle on the optimal assignment of arms. 
The work in \cite{got} takes a game-theoretic approach adapted from \cite{marden}, where the assumption of zero rewards on collisions is made and guarantees of average sub-linear regret of order $O(\log^{2 + \delta} T)$ over a time-horizon of $T$, for $0<\delta<1$, are provided. An extension of \cite{got} with near-order optimal regret of $O(\log^{1 + \delta} T)$ was presented recently in \cite{bistritz2020game}.

Another assumption in much of the prior work on multi-player MABs  that needs to be closely examined is that of zero rewards on collisions. In the example of uncoordinated dynamic spectrum access, when more than one user (player) transmits on a channel, the colliding players may receive reduced, but not necessarily zero, rates or rewards. Thus allowing for non-zero rewards on collisions results in a more realistic model. Such a setting with homogeneous reward distributions across players and non-zero rewards on collisions is considered in \cite{vvv}. The algorithm presented in \cite{vvv}, which allows for the number of players to be greater than the number of arms,  is an extension of the Musical Chairs algorithm \cite{mc} and comes with high probability guarantees of constant regret. 
\subsection{Our Contributions}
In this paper, we study a multi-player MAB with heterogeneous reward distributions and non-zero rewards on collisions. We also allow for the number of players to be greater than the number of arms. The analysis of our algorithm relies on \cite{marden}, and in contrast to the work in \cite{got}, requires non-trivial modifications to the results in \cite{marden} to accommodate non-zero rewards on collisions. To the best of our knowledge, ours is the first work to consider a model that allows for both heterogeneous reward distributions and non-zero rewards on collisions. In this setting, we propose an algorithm that achieves near order optimal regret of order $O(\log^{1 + \delta} T)$ over a time-horizon $T$, for $\delta>0$ (however small).  
\section{System Model}\label{sys}
We consider a multi-player MAB problem, in which the set of players  is $[K] = \{ 1, 2, \ldots, K\}$. The action space of each player $j \in [K]$ is the set of $M$ arms $\mathcal{A}_j = [M]$. Let the time-horizon be denoted by $T$, and the action taken (or arm played) by player $j$ at time $t$ by $a_{t,j}$. The action profile $\mathbf{a}_t$ is defined as the vector of the actions taken by the players, \textit{i.e.}, $\mathbf{a}_t = [a_{t,1},..., a_{t,K}]$. At any given time, the players can observe only their own rewards and cannot observe the actions taken by the other players. 

We assume the reward distribution of each arm to have support $[0,1]$. In the event that multiple players play the same arm $m$, they could get non-zero rewards. Let $k(a_{t,j})$ denote the number of players playing arm $a_{t,j}$ (including player $j$). Note that the number of players on arm $a_{t,j}$ is a function of the complete action profile $\mathbf{a}_t$. The reward received by player $j$ playing arm $m$, which is played by a total of $k(m)$ players (including player $j$) is denoted by $r_j(m,k(m))$. The reward is drawn from a distribution with mean 
\begin{equation}
 \mu_{j}(m,k(m)) = \expect{r_j(m,k(m))}.   
\end{equation}
We assume that
\[
\mu_{j}(m,k(m)) = 0~~\text{for all~} k(m) \geq N + 1,
\]
for some $N$ that depends on the system, i.e., when there are more than $N$ players playing the same arm, they all receive zero rewards. This constrains the maximum number of players allowed in the system to be $MN$ ($K \leq MN$).

The action space $\mathcal{A}$ of the players is simply the product space of the individual action spaces, \textit{i.e.}, $\mathcal{A} = \Pi_{j=1}^K \mathcal{A}_j$.  We refer to an element $a \in \mathcal{A}$ as a {\em matching}. Let $\mathbf{a^*} \in \mathcal{A} $ be such that
\begin{equation}
    \mathbf{a^*} \in \mathop{\arg\max}_{\mathbf{a} \in \mathcal{A}} \sum_{j=1}^K \mu_{j}(a_j, k(a_j)).
\end{equation} 

In this work, we restrict our attention to the case where there is a unique\footnote{The assumption of a unique optimal matching has been made and justified in previous works, see, e.g., \cite{got, bistritz2020game}. This assumption is needed to establish the convergence of the proposed decentralized algorithm to the optimal action profile.} {\em optimal} matching $\mathbf{a^*}$. Let $J_1 = \sum_{j=1}^K \mu_{j}(a^*_j,k(a^*_j))$ be the system reward for the optimal matching, and $J_2$ the system reward for the second optimal matching. Define
\begin{equation}
  \Delta = \frac{J_1 - J_2}{2 MN}.
\end{equation}
Unlike previous works (\cite{asilomar, dileep, got}), we do {\em not} make the assumption that the players have knowledge of $\Delta$ (or a lower bound on $\Delta$). 


The expected regret during a time horizon $T$ is defined as: 
\begin{equation}
  {R}(T) = T \sum_{j=1}^K \mu_{j}(a^*_j, k(a^*_j)) - \expect{\sum_{t=1}^T \sum_{j=1}^K \mu_{j}(a_{t,j},k(a_{t,j})) }  
\end{equation}
where the expectation is over the actions of the players. 

    In order for the players to get estimates of the mean rewards of the arms, we assume that the players have unique IDs at the beginning of the algorithm. Note that this is required only in the event that $K > M$.  If $K \leq M$, the players can get unique IDs at the beginning of the algorithm (see, e.g.,  \cite{kaufmann,asilomar}). Since all previous related works (and this work) assume that the players are time synchronized, the assumption that the players have unique IDs is justified\footnote{In the dynamic spectrum access application, time synchronization and ID assignment can be implemented via a low bandwidth side channel, which is handled, for example, by a cellular network provider.}. Note that such an assumption of unique IDs is common in applications such as multi-agent reinforcement learning \cite{boutilier1996planning}.

%
\section{Algorithm}
\begin{algorithm}
   \caption{Policy for each player $j$}
   \label{main}
\begin{algorithmic}
   \STATE {\bfseries Initialization:} Set $\hat{\mu}_{j}(m,n) = 0 $ for all $j \in [K]$, $m \in [M]$ and $n \in [N]$. Let $L_T$ be the last epoch with time horizon $T$. Parameters $\delta>0$ and $\epsilon \in (0,1)$ are provided as inputs.
   \STATE {\bfseries Calculate K}: All players run Algorithm \ref{calcK}
   \FOR{epoch $\ell = 1$ {\bfseries to} $L_T$}
   \STATE \textbf{Exploration phase:} Run Algorithm \ref{exploration} with input $\ell$
   \STATE \textbf{Matching phase:} Run Algorithm \ref{matching} with input $\ell$ for $\tau_\ell = \sqrt{c_2 \ell^{ \delta}}$ plays.  Count the number of `plays' where each action $m \in [M]$ was played that resulted in player $j$ being content: $$W^\ell(j,m) = \sum_{h = 1}^{\tau_\ell} \identityf{(a_{h,j} = m, S_{h,j} = C)}$$
   \STATE \textbf{Exploitation phase:} For $c_3 2^\ell$ time units, play the action played most frequently from epochs $\lceil \frac{\ell}{2}\rceil$ to $\ell$ that resulted in player $j$ being content: 
   $$a_j = \mathop{\arg\max}_m \sum_{i = \lceil \frac{\ell}{2}\rceil}^\ell W^i(j,m)$$
   \ENDFOR
\end{algorithmic}
\end{algorithm}

Our proposed policy for each player $j$ in the decentralized multi-player MAB setting with heterogeneous reward distributions and non-zero rewards on collisions is presented in Algorithm \ref{main}. The policy for a player depends only on the player's own actions and observed rewards. Our algorithm proceeds in epochs since we do not assume knowledge of the time horizon $T$. The parameters $\delta >0$, and  $\epsilon \in (0, 1)$ are inputs to the algorithm, and further details on these parameters are provided in Sections \ref{main_result} and \ref{matchingalg} respectively. Let $L_T$ denote the number of epochs in time horizon $T$. Each epoch $\ell$ has three phases: Exploration, Matching and Exploitation.

Since we assume that the players have been assigned unique IDs at the beginning of the algorithm, it may be reasonable to assume that the total number of players is also known to them. Nevertheless, for the sake of completeness, we provide Algorithm \ref{calcK}, which the players can use to calculate the number of players before the beginning of epoch 1. Since the players have unique IDs from $1$ to $K$, players form groups of $N$, and use the fact that zero rewards are received when more than $N$ players occupy a channel, to calculate the number of complete groups of $N$, and if needed, the size of the last incomplete group. 

\begin{algorithm} 
    \caption{Calculate K}
    \label{calcK}
\begin{algorithmic}
    \STATE Starting from player number 2, players form groups of size $N$ in order of their IDs (from 2 to $N$+1, $N+2$ to $2N+1$, etc.), giving $\lceil \frac{K-1}{N} \rceil \leq M$ groups. The last group may have size $<N$.
    \STATE {\bfseries Calculate number of groups} ($M$ time units): Groups numbered 1 to $\lceil \frac{K-1}{N} \rceil$ occupy arm corresponding to their group number. Player 1 plays arms $1$ to $M$ in order. Player 1 then knows the number of complete groups of size $N$ (number of arms that yield zero reward).
    \STATE {\bfseries Calculate number of players in last group} ($N$ time units): If a group received non-zero reward (an incomplete last group if exists) during the previous step, the players of that group and player 1 play arm $1$ for $N$ time units. Players of group $1$ ($2$ to $N+1$) pick arm 1 in order cumulatively, \textit{i.e.}, player $2$ plays arm $1$ during the first time unit of this step, players $2$ and $3$ during the second, and so on. The number of players in last group $N_1$ is equal to $N$ minus the number of additional players from group 1 needed for players on arm 1 to get zero reward. At the end of this step, player 1 knows $K$ and conveys it to other players in the next stages.
    \STATE {\bfseries Convey $K$ to complete groups} ($(K+1)M$ time units): Groups numbered 1 to $\lceil \frac{K-1}{N} \rceil$ occupy arm corresponding to their group number. Player 1 plays arms 1 to $M$ in order $K$ times ($KM$ time units), during which the players in the complete groups receive zero rewards. Player 1 then stays silent for the next $M$ time units, which conveys $K$ to the complete groups. The complete groups now know if there is an incomplete group or not, and release their arms.
    \STATE {\bfseries Convey $K$ to incomplete group} ($K + 1$ time units): If there is no incomplete group, all players stay silent for this step. Otherwise, the first $N_1+1$ players of group 1 (who know $K$ now) play the incomplete group arm for $K$ time units, and stay silent for the next time unit.
\end{algorithmic}{}
\end{algorithm}{}

The exploration phase is for player $j$ to obtain estimates of the mean rewards (denoted by $\hat{\mu}_j(m,n)$) of arms $m \in [M]$, for all $n \in [N]$. This phase proceeds for a fixed number of time units ($T_e$) in every epoch. Since every player has an unique ID, and the total number of players has been calculated, the exploration phase follows a protocol where the players sample each arm $m\in [M]$, for each $n\in [N]$, for $T_0 \ell^\delta$ time units. The exploration phase of epoch $\ell$ proceeds for $T_e \leq KMN T_0 \ell^\delta$ time units (the inequality arises as the number of players may not be an exact multiple of $N$). During epoch $\ell$, if the estimated mean rewards obtained at the end of the exploration phase  do not deviate from the mean rewards by more than $\Delta$, the exploration phase of this epoch is \emph{successful}. Since we do not assume that the players have knowledge of $\Delta$, having increasing lengths of the exploration phases by setting $\delta > 0$ ensures that this phase is successful with high probability eventually (for large enough $\ell$). 

Once the players have estimates of the mean rewards of the arms, they all need to find the action profile that maximizes the system reward. This is done in the matching phase of each epoch. Given an action profile $\mathbf{a}$, we define the utility of player $j$ to be 
\begin{equation} \label{eq:utildef}
    u_j(\mathbf{a}) = \hat{\mu}_j(a_j, k(a_j)).
\end{equation} 
Section \ref{matchingalg} explains in detail the matching phase and the choice of the above defined utility function. The matching phase is based on the work in \cite{marden}, in which a strategic form  game is studied, where players are aware of only their own payoffs (utilities). A decentralized strategy is presented in \cite{marden} that leads to an efficient configuration of the players' actions. The work in \cite{got}, which considers essentially the same setting as in the current paper but with zero rewards on collisions, applies the decentralized strategy proposed in \cite{marden} directly without any modifications. This is because, with zero rewards on collisions, the utility of each player $j$ can take only two possible values: (i) $0$,  or (ii) the estimated mean reward of the arm chosen, which is known to player $j$ from the exploration phase.
Therefore, the matching phase in \cite{got} corresponds precisely to the problem studied in \cite{marden}, where the players know their own utilities exactly. However, in our setting with non-zero rewards on collisions, we face the difficulty that the utilities of the players take on more than two possible values and are not known exactly. This is because, in our setting, each player $j$ does not know the total number of players choosing the same arm $k(a_j)$, and therefore cannot determine the utility defined in \eqref{eq:utildef} exactly. Each player $j$ needs to \emph{estimate} $k(a_j)$ based on the actual instantaneous rewards seen in the matching phase, as described in Algorithm \ref{matching}.
Thus, we have to work with \emph{estimated} utilities as opposed to exact utilities as in \cite{marden} and \cite{got}. In order to provide regret guarantees, we need to prove that our algorithm leads to an action profile maximizing the sum of utilities of the players even with these estimated utilities. This is analyzed in Section \ref{matchingalg}. 
This phase proceeds for $\tau_\ell \tilde{\tau}_\ell = c_2  \ell ^{ \delta}$ time units in epoch $\ell$, where $c_2$ is a constant.  As in the exploration phase, we need increasing lengths of matching phases ($\delta > 0$) to guarantee that the players identify the optimal action profile with high probability. 

The action profile identified at the end of the matching phase is played in the exploitation phase for $c_3 2^\ell$ time units, where $c_3$ is a constant. As $\ell$ increases, the players get better estimates of the mean rewards and the probability of identifying the optimal action profile increases. Therefore, the length of the exploitation phase is set to be exponential in $\ell$. 
The constants $c_2$ and  $c_3$ are are chosen to be of the order of $T_e$, the time taken by the exploration phase.

Below, we provide the pseudo-code of the protocol to calculate $K$ (Algorithm \ref{calcK}), and the exploration phase (Algorithm \ref{exploration}) and the matching phase (Algorithm \ref{matching}) of the algorithm.

\begin{algorithm} [!t]
   \caption{Exploration Phase}
   \label{exploration}
\begin{algorithmic}
   \FOR{$n = 1$ to $N$}
   \STATE Starting from player 1, players form groups of size $n$ in order of their IDs (from $1$ to $n$, $n+1$ to $2n$, and so on). Let the number of complete groups of size $n$ be $G$.
   \IF{$\lfloor \frac{G}{M} \rfloor \geq 1$} 
          \FOR{$g = 1$ to $\lfloor \frac{G}{M} \rfloor$}
           \STATE Groups $(g-1)M +1$ to $gM$ play arms $1$ to $M$ in a round-robin fashion for $T_0 \ell^\delta$ time units, \textit{i.e.}, they play arms $1, 2, \ldots, M$ respectively for $T_0 \ell^\delta$ time units, then play arms $2, 3, \ldots, M, 1$ respectively for $T_0 \ell^\delta$ time units, and so on, until they play arms $M, 1, 2, \ldots, M-1$ respectively for  $T_0 \ell^\delta$ time units.
           \ENDFOR
    \ENDIF
    \IF{$\frac{G}{M} \notin \mathbb{Z}$}
        \STATE Groups $\lfloor \frac{G}{M} \rfloor M +1$ to $G$ play arms $1$ to $M$ in a round-robin fashion for $T_0 \ell^\delta$ time units, \textit{i.e.}, they play arms $1, \ldots, G - \lfloor \frac{G}{M} \rfloor M$ respectively for $T_0 \ell^\delta$ time units, then play arms $2, \ldots , G - \lfloor \frac{G}{M} \rfloor M + 1$ for $T_0 \ell^\delta$ time units, and so on, until they play arms $M , 1 , \ldots,  G - \lfloor \frac{G}{M} \rfloor M - 1$ for $T_0 \ell^\delta$ time units.  
    \ENDIF
   %
   \STATE If the final group is incomplete with $n_1<n$ players, it is completed with $n-n_1$ players from group 1, and the completed group plays arms 1 to $M$ for $T_0 \ell^\delta$ time units each.
   \ENDFOR
\end{algorithmic}
\end{algorithm}{}

\begin{algorithm} 
   \caption{Matching phase algorithm}
   \label{matching}
\begin{algorithmic}
   \STATE {\bfseries Initialization:} Let $\kappa > MN$. Denote by $\hat{Z}_{h,j}$ the observed (estimated) state of player $j$ at time $h$, and set  $\hat{Z}_{1,j} = [\Bar{a}_{1,j}, \Bar{u}_{1,j}, S_{1,j}],$ where $\Bar{a}_{1,j}\mathop{\sim}\limits^{\text{unif}}[M]$, $\Bar{u}_{1,j} = 0$ and $S_{1,j} = D$. Set $\tau_\ell = \tilde{\tau}_\ell = \sqrt{c_2 \ell^\delta}$.  Input parameter $\epsilon \in (0,1)$. 
   \FOR{play $h = 1$ {\bfseries to} $\tau_\ell$}
   \STATE \textbf{Action dynamics:}
   \STATE If $S_{h,j} = C$, set action $a_{h,j}$ as:
   \[
   a_{h,j} = \begin{cases} \Bar{a}_{h,j} &\text{with prob} \;\; 1 - \epsilon^\kappa \\ a \in [M]\setminus{\Bar{a}_{h,j}}  &\text{with uniform prob} \;\; \frac{\epsilon^\kappa}{M-1}. \end{cases}
   \]
   \STATE If $S_{h,j} = D$, action $a_{h,j}$ is chosen uniformly from $[M]$.
   \STATE \textbf{Estimate utility:} Upon choosing action $a_{h,j}$, play that arm for $\tilde{\tau}_\ell$ time units, and let the sample mean of the rewards observed during this duration be $\Bar{r}(a_{h,j})$. If $\Bar{r}(a_{h,j}) = 0$, the estimated utility of the player $\hat{u}_{h,j} = 0$. Else let 
   \[
   \hat{k}(a_{h,j}) = \mathop{\arg\min}_{n \in [N],\hat{\mu}_{j}(a_{h,j},n) \neq 0 } |\Bar{r}(a_{h,j}) - \hat{\mu}_{j}(a_{h,j},n)|.
   \]
   and the estimated utility $\hat{u}_{h,j}$ is: 
   \[
   \hat{u}_{h,j} = \hat{\mu}_{j}(a_{h,j},\hat{k}(a_{h,j})).
   \]
   \STATE \textbf{State Dynamics:} 
   \STATE If $S_{h,j} = C$ and $[a_{h,j},\hat{u}_{h,j}] = [\Bar{a}_{h,j}, \Bar{u}_{h,j}]$, set: 
   $$\hat{Z}_{h+1,j} = \hat{Z}_{h,j}$$
   \STATE If $S_{h,j} = C$ and $[a_{h,j}, \hat{u}_{h,j}] \neq [\Bar{a}_{h,j}, \Bar{u}_{h,j}]$ or $S_{h,j} = D$, set:
   \begin{equation}\label{state}
       \hat{Z}_{h+1,j} = \begin{cases} [a_{h,j}, \hat{u}_{h,j}, C] &\text{with prob} \;\; \epsilon^{1 - \hat{u}_{h,j}} \\ [a_{h,j}, \hat{u}_{h,j}, D] &\text{with prob} \;\; 1 - \epsilon^{1 - \hat{u}_{h,j}} \end{cases} 
   \end{equation}
   \ENDFOR
\end{algorithmic}
\end{algorithm}

\section{Main Result}\label{main_result}
\begin{theorem} \label{thm:main_th}
Given the system model specified in Section 2, the expected regret of the proposed algorithm for a time-horizon $T$ and some $0< \delta <1 $ is $R(T) = O(\log^{1+\delta} T)$.
\end{theorem}
\begin{IEEEproof}
Let $L_T$ be the last epoch within a  time-horizon of  $T$. The regret incurred during the $L_T$ epochs can be analyzed as the sum of the regret incurred during the three phases of the algorithm. The exploitation phase in epoch $\ell$ of the algorithm lasts for $c_3 2^{\ell}$ time units. Thus, it is easy to see that $L_T < \log{T}$. Let $R_1$, $R_2$ and $R_3$ denote the regret incurred  over $L_T$ epochs, during the exploration phase, the matching phase, and the exploitation phase, respectively. Let $\ell_0$ be the first epoch $\ell$ such that 
\begin{align}
    \frac{T_0 \Delta^2}{2} {\left(\frac{\ell}{4}\right)}^{\delta} & \geq 1, \label{exp1} \\
    C_0 e^{-C_\rho \ell^{\delta/2}} &\leq e^{-1} \label{match1}\\
   \tilde{\tau}_\ell =  \sqrt{c_2 \ell^\delta} & \geq \lceil \frac{2 \ln({\frac{2}{\epsilon^\kappa}})}{(\Delta + 2\nu_{min})^2} \rceil \label{match_2}
\end{align}
where $C_\rho$ is defined in \eqref{crho} in Appendix~\ref{sec:app_B}. Note that \eqref{exp1} - \eqref{match_2} hold for all $\ell \geq \ell_0$. We upper bound the regret incurred during the exploration and matching phases by $K$ (which is the maximum regret that can be accumulated in one time unit) times the total time taken by these phases. For the exploitation phase, we incur regret only when the action profile played during this phase is not optimal.

\begin{enumerate}
    \item Exploration phase: Since the exploration phase in each epoch $\ell$ proceeds for at most $KMNT_0 \ell^\delta$ time units,
    \begin{equation}
        R_1 \leq \sum_{\ell = 1}^{L_T} K^2 M N T_0 \ell^\delta \leq K^2 M N T_0 \log^{1 + \delta} T.
    \end{equation}
    \item Matching phase: In epoch $\ell$, the matching phase runs for $\tau_\ell \tilde{\tau}_\ell = c_2  \ell^{\delta}$ time units. Thus 
    \begin{equation}
        \begin{split}
            R_2 &\leq  K \sum_{\ell = 1}^{L_T} c_2  \ell^{\delta} \leq K c_2  L_T^{1 + \delta} \leq K c_2  \log^{1+\delta}{T}.
        \end{split}
    \end{equation}
    \item Exploitation phase: In the exploitation phase, regret is incurred in the following two events:
    \begin{enumerate}
        \item Let $E^\ell$ denote the event that there exists some player $j \in [K]$, arm $m \in [M]$, and number of players on the arm $n \leq N$, such that there exists some epoch $i$, with $\lceil \frac{\ell}{2} \rceil \leq i \leq \ell$, such that the estimate of the mean reward $\hat{\mu}_j(m,n)$ obtained after the exploration phase of epoch $i$ satisfies $|\hat{\mu}_j(m,n) - \mu_{j}(m,n)| \geq \Delta$.  
        \item Let $F^\ell$ denote the event that given that all the players have $|\mu_{j}(m,n) - \hat{\mu}_{j}(m,n)| \leq \Delta$ for all $m \in [M]$ and all $n \in [N]$ for all epochs $\lceil \frac{\ell}{2} \rceil $ to $ \ell$, the action profile chosen in the matching phase of epoch $\ell$ is not optimal. 
    \end{enumerate}
    From Lemma \ref{lem1}, we have an upper bound on the probability of event $E_{j,m,n}^\ell$ that for some fixed player $j$, arm $m$, and number of players on the arm $n$, there exists some epoch $i$, with $\lceil \frac{\ell}{2} \rceil \leq i \leq \ell$, such that the estimate of the mean reward $\hat{\mu}_j(m,n)$ obtained after the exploration phase of epoch $i$ satisfies $|\hat{\mu}_j(m,n) - \mu_{j}(m,n)| \geq \Delta$. Thus, we have that
    \begin{align}
        P(E^\ell) &= P\left( \underset{j \in [K], m\in [M], n \in [N]}{\bigcup} E_{j,m,n}^\ell \right) \\
        &\leq KMN P(E_{j,m,n}^\ell) \\
        &\leq \frac{KMNe^{- \frac{T_0 \Delta^2}{2} (\frac{\ell}{4})^{\delta} \ell}}{1 - e^{-T_0 \Delta^2 (\frac{\ell}{4})^{\delta}}}.\label{exp_prob}
    \end{align}
   We also have the following upper bound on the probability of event $F^\ell$ from Lemma \ref{lem4}:
    \begin{equation}\label{match_prob}
        P(F^\ell) \leq \left(C_0 \exp{(-C_\rho\ell^{\delta/2})}\right)^\ell.
    \end{equation} 
 Using \eqref{exp1} and \eqref{match1} and the upper bounds in \eqref{exp_prob} and \eqref{match_prob}, we have that for all epochs $\ell \geq \ell_0$:
 \begin{align}
     P(E^\ell) &\leq 2KMNe^{-\ell} \label{exp_uppbd}\\
     P(F^\ell) &\leq  e^{-\ell} \label{match_uppbd}.
 \end{align}
    Therefore,
    \begin{align}
            R_3 &= K\sum_{\ell = 1}^{L_T} c_3 2^\ell (P(E^\ell) + P(F^\ell)) \label{exploit_reg}\\
            &\leq 2Kc_32^{\ell_0} + K c_3 \sum_{\ell = \ell_0}^{L_T}  (2KMN + 1) \left(\frac{2}{e}\right)^\ell    \label{upp_bound} \\
            &\leq 2Kc_32^{\ell_0} + \frac{2Kc_3 ( 2KMN+ 1)}{e-2}.
    \end{align}
    
    Note that \eqref{exploit_reg} follows from the fact that regret is incurred in the exploitation phase only if events $E^\ell$ or $F^\ell$ occur, \eqref{upp_bound} follows from upper bounding the regret incurred during the first $\ell_0$ epochs by the maximum possible regret, and using \eqref{exp_uppbd} and \eqref{match_uppbd} to upper bound the regret for epochs greater than $\ell_0$. 
    
\end{enumerate}

Thus 
\begin{equation}
    \begin{split}
    R(T) &= R_1 + R_2 +R_3 \\
    &\leq K^2 M N T_0 \log^{1 + \delta} T + K c_2 \log^{1+\delta}{T} \\ 
    &+ 2Kc_32^{\ell_0} + \frac{2Kc_3 ( 2KMN+ 1)}{e-2}  \\
    &\sim O(\log^{1+\delta} T).
    \end{split}
\end{equation}
\end{IEEEproof}

\begin{remark}
Since the system parameters such as $\Delta$ and the mixing times of the Markov chain are unknown, having increasing lengths of exploration and matching phases by setting $\delta >0$ guarantees that there exists some epoch $\ell_0$, such that for all epochs $\ell \geq \ell_0$, the probabilities of events $E^\ell$ and $F^\ell$ decrease exponentially as $e^{-\ell}$.  
However, we have observed empirically 
(refer Section \ref{experiments}) that setting $\delta = 0$ results in incurring zero regret during the exploitation phase with high probability right from the earlier epochs. 
\end{remark}

\begin{remark}
If a lower bound on the parameter $\Delta$, say $\tilde{\Delta}$, is known, $T_0$ can be set to $\frac{2}{\tilde{\Delta}^2}$. Note that, since $c_2$ and $c_3$ are of order $T_e$, 
and $T_e$ is of the order $\frac{KMN}{\Delta^2}$, the regret bound in terms of all the key system parameters is then of order $O(\frac{K^2MN}{\tilde{\Delta}^2} \log^{1+\delta}T)$. 
\end{remark}

In the following sections, we provide details on the key results that are used in the proof of Theorem~\ref{thm:main_th}.
 
\section{Exploration Phase}\label{explore}

The exploration phase is for player $j$ to obtain estimates of the mean rewards of arms $m \in [M]$ for all $n \in [N]$. During the exploitation phase of epoch $\ell$, each player plays the most frequently played action during the matching phases of epochs $\lceil \frac{\ell}{2} \rceil$ to $\ell$ while being in a content state. Thus, we bound the probability that the estimated mean rewards deviate from the mean rewards by more than $\Delta$ in at least one of the epochs from $\lceil \frac{\ell}{2} \rceil$ to $\ell$, by $O(e^{-\ell})$.

\begin{lemma}\label{lem1}
Given $\Delta$ as defined in Section 2, for a fixed player $j$, arm $m$, and number of players on the arm $n \leq N$, let $E_{j,m,n}^\ell$ denote the event that there exists at least one epoch $i$ with $\lceil \frac{\ell}{2} \rceil \leq i \leq \ell$ such that the estimate of the mean reward $\hat{\mu}_j(m,n)$ obtained after the exploration phase of epoch $i$ satisfies $|\hat{\mu}_j(m,n) - \mu_{j}(m,n)| \geq \Delta$. Then  
\begin{equation}
    P(E_{j,m,n}^\ell) \leq  \frac{e^{- \frac{T_0 \Delta^2}{2} (\frac{\ell}{4})^{\delta} \ell}}{1 - e^{-T_0 \Delta^2 (\frac{\ell}{4})^{\delta}}}.
\end{equation}
\end{lemma}
The proof of this lemma is provided in Appendix \ref{app:exploration}.
\section{Matching Phase}\label{matchingalg}
Strategic form games in game theory are used to model situations where players choose actions simultaneously (rather than sequentially) and do not have knowledge of the actions of other players. In such games, each player has a utility function $u_j: \mathcal{A} \to [0,1]$, that assigns a real valued payoff (utility) to each action profile $\mathbf{a} \in \mathcal{A}$. An algorithm that works under the assumption that every agent can observe only their own action and utility received is called a payoff based method. The matching phase of our proposed algorithm builds on \cite{marden}, where a payoff based decentralized algorithm that leads to maximizing the sum of the utilities of the players is presented. In order to pose the multi-player MAB problem as a strategic form game, we need to design the utility functions of the players in a way such that the system regret is minimized, or equivalently, the system performance is maximized. Denote by $u_j(\mathbf{a})$ the utility of player $j$ associated with the action profile $\mathbf{a}$. We define: 
\begin{equation}
  u_j(\mathbf{a}) = \hat{\mu}_j(a_j, k(a_j)).
\end{equation}
A similar utility function is used in \cite{got}, where it is assumed that collisions result in zero rewards for the colliding players. Note that when players receive zero rewards on collisions, $\hat{\mu}_j(a_j, k(a_j)) = 0$ whenever $k(a_j) \geq 2$. Therefore, in the work of \cite{got},  the utility for each player $j$ can be determined exactly based on whether the instantaneous reward is zero or non-zero. 

However, in the setting we consider with non-zero rewards on collision, since player $j$ does not know $k(a_j)$, the utility  $u_j(\mathbf{a})$ is also not known exactly. 
Each player $j$ needs to \emph{estimate} $k(a_j)$ based on the instantaneous rewards seen in the matching phase, and use this to estimate the utility, as described in Algorithm~\ref{matching}.

The action profile that maximizes the sum of the utilities is called an efficient action profile. The following lemma states that if for all $j\in [K], m \in [M]$ and $n \in [N]$,  $|\hat{\mu}_j(m,n) - \mu_{j}(m,n)| \leq \Delta$, then by our choice of the utility function, the efficient action profile maximizing the sum of the utilities is also the same as the optimal action profile maximizing the sum of expected rewards. 

\begin{lemma}
    If for all $j\in [K], m \in [M]$ and $n \in [N]$, the following condition is satisfied : 
    \begin{equation}\label{err}
        |\hat{\mu}_j(m,n) - \mu_{j}(m,n)| \leq \Delta,
    \end{equation} 
    then 
    \begin{equation}
        \mathop{\arg\max}_{\mathbf{a} \in \mathcal{A}} \sum_{i=1}^{K} \mu_j(a_j,k(a_j)) = \mathop{\arg\max}_{\mathbf{a} \in \mathcal{A}} \sum_{i=1}^{K} \hat{\mu}_j(a_j,k(a_j)).
    \end{equation}
\end{lemma}

The condition given by (\ref{err}) is guaranteed with high probability by the exploration phase of the algorithm. The proof of the above lemma is similar to the proof of \cite[Lemma 1]{got}. Thus, the efficient action profile that maximizes the sum of utilities (estimated mean rewards) is the same as the optimal action profile that minimizes regret or equivalently, maximizes system performance.

\subsection{Description of The Matching Phase Algorithm}

This phase consists of $\tau_\ell$ plays, where each play lasts $\tilde{\tau}_\ell$ time units and $\epsilon \in (0,1)$ is a parameter of the algorithm. Each player $j$ is associated with a state $Z_{h,j} = [\Bar{a}_{h,j}, \Bar{u}_{h,j}, S_{h,j}]$ during play $h$, where $\Bar{a}_{h,j} \in [M]$ is the baseline action of the player, $\Bar{u}_{h,j} \in [0,1]$ is the baseline utility of the player and $S_{h,j} \in \{C,D\}$ is the mood of the player ($C$ denotes \enquote{content} and $D$ denotes \enquote{discontent}). Note that $\{Z_{1,j}, Z_{2,j}, ...\}$ are the states resulting from running the matching phase algorithm (essentially the state update step) with the exact utilities $u_j$.
Since in our setting, the utility or payoff received by each player is estimated and the state is updated using this estimate, each player in our algorithm works instead with an estimated state $\hat{Z}_{h,j}$. 

When the player is content, the baseline action is chosen with high probability ($1 - \epsilon^\kappa$) and every other action is chosen with uniform probability. The parameter $\epsilon$ is provided as an input to the algorithm, and Appendix \ref{app:exp} discusses how to choose $\epsilon$. If the player is discontent, the action is chosen uniformly from all arms and there is a high probability that the player would choose an arm different from the baseline action. This part of the algorithm constitutes the action dynamics. The baseline action can be interpreted as the arm the agent expects to play for a long time eventually and the baseline utility can be interpreted as the payoff the player expects to receive upon playing the baseline action. The player being content is an indication that the payoff received by the player while playing his baseline action is satisfactory and as expected. Thus, the goal in designing the matching phase algorithm is for all the players to align their baseline actions and baseline utilities to the efficient action profile and be content in this state. Note that the action dynamics do not depend on the utilities.

We have seen the justification for using the utility function
$$u_j(\mathbf{a}) = \hat{\mu}_j(a_j, k(a_j))$$ in the introduction of Section \ref{matchingalg}. However, the players observe only their own instantaneous reward and do not know $k(a_{j})$ (since it depends on the actions chosen by all the players during the action dynamics step), in order to determine the utility. Thus, each player estimates $k(a_{j})$ as $\hat{k}(a_{j})$ and uses this to estimate $u_j(\mathbf{a})$ as $\hat{u}_j(\mathbf{a})$. This is done by the player pulling the arm chosen during the action dynamics step for $\tilde{\tau}_\ell$ time units and recording the sample mean of the rewards observed during this duration as $\Bar{r}_j(a_j)$. The estimate $\hat{k}(a_j)$ is given by:
$$\hat{k}(a_j) = \mathop{\arg\min}_{n \in [N],\hat{\mu}_{j}(a_j,n) \neq 0 } |\Bar{r}(a_j) - \hat{\mu}_{j}(a_j,n)|$$
and $\hat{u}_j(\mathbf{a}) = \hat{\mu}_{j}(a_j,\hat{k}(a_j))$. 

\begin{lemma}\label{lem2}
If $\tilde{\tau}_\ell \geq \lceil \frac{2 \ln({\frac{2}{\epsilon^\kappa}})}{(\Delta + 2\nu_{min})^2} \rceil$, which holds for all $\ell \geq \ell_0$ (see \eqref{match_2}), we have that 
\[
p_\epsilon =  P\{u_j(\mathbf{a}) \neq \hat{u}_j(\mathbf{a})\} \leq \epsilon^\kappa,
\]
where 
\[
\nu_{\mathrm{min}} = \min_{j,m,n_1,n_2} |\mu_{j}(m,n_1) - \mu_{j}(m,n_2)| ,
\]
with $n_1,n_2 \in [N]$, $\mu_{j}(m,n_1), \mu_{j}(m,n_2) \neq 0$, $j \in [K]$ and $m \in [M]$.
\end{lemma}
The proof follows directly from Hoeffding's inequality.


Thus, each player has an estimate of his utility that is correct with high probability. Note that in Algorithm \ref{matching}, $\hat{u}_j(\mathbf{a})$ is referred to as just $\hat{u}_j$ for readability. 

The player updates his current estimated state by comparing the action played and the estimate of the utility received with the baseline action and baseline utility associated with the current estimated state. If the player is content and his baseline action and utility match the action played and the estimate of the utility, the estimated state remains the same. Otherwise, the estimated state for the next play is chosen probabilistically based on the estimate of the utility. The rationale behind the particular probabilities chosen is that when the utility received is high, the player is more likely to be content.

The utility each player receives is equivalent to feedback from the system on how the entire action profile affects the reward received by this player. If the player receives a lower payoff due to that arm not being good or due to collisions, there is a higher probability of the player becoming discontent and exploring other arms. On the other hand, if the payoff received is higher, there is a higher probability of the player staying content and exploiting the same arm again. Thus the agent dynamics and state dynamics balance the exploration-exploitation tradeoff in the multi-player MAB setting.

During the matching phase algorithm, each player keeps a count of the number of times each arm was played that resulted in the player being content:
$$W^\ell(j,m) = \sum_{h = 1}^{\tau_\ell} \identityf{(a_{h,j} = m, S_{h,j} = C)}.$$
The action chosen by the player for the exploitation phase is the arm played most frequently from epochs $\lceil \frac{\ell}{2}\rceil$ to $\ell$ that resulted in the player being content: 
   $$a_j = \mathop{\arg\max}_m \sum_{i = \lceil \frac{\ell}{2}\rceil}^\ell W^i(j,m).$$
 
\subsection{Analysis of the Matching Phase Algorithm}

The matching phase algorithm is based on the work in \cite{marden}, and the guarantees provided there state that the action profile maximizing the sum of the utilities of the players is played for a majority of the time. The analysis of the algorithm in \cite{marden} relies on the theory of regular perturbed Markov decision processes \cite{young}.

The dynamics of the matching phase algorithm induce a Markov chain over the state space $\mathcal{Z} = \Pi_{j=1}^K ([M] \times [0,1] \times \mathcal{M})$ where $\mathcal{M} = \{C,D\}$, $\textit{i.e.}$, each state $z \in \mathcal{Z}$ is a vector of the states of all players. Let $P^0$ denote the probability transition matrix of the process when $\epsilon = 0$ and $P^\epsilon$ denote the transition matrix when $\epsilon > 0$. The process $P^\epsilon$ is a regular perturbed Markov process if for any $z,z' \in \mathcal{Z}$ (Equations (6),(7) and (8) of Appendix of \cite{young}): 
\begin{enumerate}\label{conditions}
    \item $P^\epsilon$ is ergodic
    \item $\lim_{\epsilon \to 0} P^\epsilon_{zz'} = P^0_{zz'} $
    \item $P^\epsilon_{zz'} > 0$ implies for some $\epsilon$, there exists $r \geq 0$ such that $0< \lim_{\epsilon \to 0} \epsilon^{-r} P^\epsilon_{zz'} <\infty$
\end{enumerate}
The value of $r$ satisfying the third condition is called the resistance of the transition $z \to z'$, denoted by $r(z \to z')$.

Let $\mu^\epsilon$ be the unique stationary distribution of $P^\epsilon$, where $P^\epsilon$ is a regular perturbed Markov process. Then $\lim_{\epsilon \to 0} \mu^\epsilon$ exists and the limiting distribution $\mu^0$ is a stationary distribution of $P^0$. The stochastically stable states are the support of $\mu^0$. The main result of \cite[Theorem 3.2]{marden} states that the stochastically stable states of the Markov chain induced by their proposed payoff based decentralized learning rule maximize the sum of the utilities of the players. However,  \cite[Theorem 3.2]{marden} cannot be applied directly in our setting, because:
\begin{enumerate}
    \item The utilities in our algorithm are estimated, and thus the resulting dynamics in our matching phase algorithm are different from that in \cite{marden}.
    \item Our game is not interdependent (interdependence property implies that  it is not possible to divide the agents into two distinct subsets, where the actions of agents in one subset do not affect the utilities of those in the other).
\end{enumerate}

We prove that, despite these two differences, a result similar to \cite[Theorem 3.2]{marden} holds, \textit{i.e.}, the stochastically stable states of our matching phase algorithm maximize the sum of the utilities of the players. Following that, we prove that the stochastically stable state that maximizes the sum of utilities is played for the majority of time in the matching phase with high probability.

\begin{theorem}\label{thm2}
Under the dynamics defined in Algorithm \ref{matching}, a state $z \in \mathcal{Z}$ is a stochastically stable state if and only if the action profile given by the baseline actions of all the players in this state maximizes the sum of their utilities and all the players are content.
\end{theorem}

Below, we present a few key lemmas required for the proof of Theorem \ref{thm2}. The rest of the proof follows from the proof of \cite[Theorem 3.2]{marden}.

\begin{lemma}\label{perturbed}
The dynamics presented in Algorithm \ref{matching} is a regular perturbed Markov process.
\end{lemma}

\begin{IEEEproof}
The transitions where $P^\epsilon_{zz'} > 0$ and $\lim_{\epsilon \to 0} P^\epsilon_{zz'}  = 0$ are called $\epsilon$ perturbations. The players in our algorithm do not directly observe their utilities. Instead they estimate their utilities, and there is a probability of error of $p_\epsilon$ in this step. However, this can be represented as an $\epsilon$ perturbation by rewriting the state update step as follows:

\noindent 
If $S_{h,j} = C$:

\noindent 
If $[a_{h,j}, {u}_{h,j}] = [\Bar{a}_{h,j}, \Bar{u}_{h,j}]$, the new state is 
    \begin{equation}
    \begin{split}
    &\hat{Z}_{h+1,j} = [\Bar{a}_{h,j}, \Bar{u}_{h,j}, C] \to \\
        &\begin{cases}
      [\Bar{a}_{h,j},\Bar{u}_{h,j},C] & \text{w.p. }\ 1 - p_\epsilon \\
      [\Bar{a}_{h,j},\hat{u}_{h,j},C], & \text{w.p. }\ p_\epsilon(\epsilon^{1-\hat{u}_{h,j}}) \\
      [\Bar{a}_{h,j},\hat{u}_{h,j},D], & \text{w.p. }\ p_\epsilon(1 - \epsilon^{1-\hat{u}_{h,j}}) 
    \end{cases}
    \end{split}
    \end{equation}
    If $[a_{h,j}, {u}_{h,j}] \neq [\Bar{a}_{h,j}, \Bar{u}_{h,j}]$, where $q < 1$ 
    \begin{equation}
    \begin{split}
    &\hat{Z}_{h+1,j} = [\Bar{a}_{h,j}, \Bar{u}_{h,j}, C] \to \\
        &\begin{cases}
      [{a}_{h,j},{u}_{h,j},C], & \text{w.p. }\ (1 - p_\epsilon)(\epsilon^{1-{u}_{h,j}}) \\
      [{a}_{h,j},{u}_{h,j},D], & \text{w.p. }\ (1 - p_\epsilon)(1 - \epsilon^{1-u_{h,j}}) \\
      [\Bar{a}_{h,j},\Bar{u}_{h,j},C] & \text{w.p. }\  qp_\epsilon \\
      [{a}_{h,j},\hat{u}_{h,j},C], & \text{w.p. }\ (1 - q)p_\epsilon(\epsilon^{1-\hat{u}_{h,j}}) \\
      [{a}_{h,j},\hat{u}_{h,j},D], & \text{w.p. }\ (1 - q)p_\epsilon(1 - \epsilon^{1-\hat{u}_{h,j}})
    \end{cases}
    \end{split}
    \end{equation}

\noindent
If $S_{h,j} = D$:

\begin{equation}
\begin{split}
    &\hat{Z}_{h+1,j} =[\Bar{a}_{h,j}, \Bar{u}_{h,j}, D] \to \\
        &\begin{cases}
      [{a}_{h,j},{u}_{h,j},C] & \text{w.p. }\ (1 - p_\epsilon) \epsilon^{1-u_{h,j}} \\
      [{a}_{h,j},{u}_{h,j},D], & \text{w.p. }\ (1 - p_\epsilon)(1 - \epsilon^{1-u_{h,j}}) \\
      [{a}_{h,j},\hat{u}_{h,j},C], & \text{w.p. }\  p_\epsilon \epsilon^{1-\hat{u}_{h,j}}  \\
      [{a}_{h,j},\hat{u}_{h,j},D], & \text{w.p. }\  p_\epsilon (1 - \epsilon^{1-\hat{u}_{h,j}})  
    \end{cases}
    \end{split}
    \end{equation}
    
Note that $Z_{h+1,j}$ is the state obtained when updated with the true utility or payoff received at each time. Another way of looking at this transformed state dynamics is that, 
\begin{equation}
    \hat{Z}_{h+1,j} = {Z}_{h+1,j}  \;\;\; \text{w.p. }\ 1 - p_\epsilon 
\end{equation}
    
We have from Lemma \ref{lem2} that $p_\epsilon \leq \epsilon^\kappa$. Thus we can see that the unperturbed process of our matching algorithm (\textit{i.e.} when $\epsilon = 0$) is the same as that in \cite{marden}.

It can be easily seen from the rewritten state update step that our dynamics satisfy the three conditions (mentioned in the beginning of Section \ref{conditions}) for a regular perturbed Markov chain.

\end{IEEEproof}

The second way in which our dynamics differ from those in \cite{marden} is that our strategic form game is not interdependent (see \cite[Definition 1]{marden}). The interdependence property implies that  it is not possible to divide the agents into two distinct subsets, where the actions of agents in one subset do not affect the utilities of those in the other. However, in our case, consider an action profile where $N + 1$ players play an arm $m$. These $N+1$ players will receive zero utility, no matter what the other players play. Thus, our game is not interdependent. However, the only time the interdependence property is used in \cite{marden} is to find the recurrence classes of $P^0$, and we can prove a similar result about the recurrence classes of the Markov chain induced by our matching phase algorithm using the  specific structure of our algorithm. 
\begin{lemma}\label{recurring}
Let $D^0$ represent the set of states in which everyone is discontent. Let $C^0$ represent the set of states in which all agents are content and their baseline actions and utilities are aligned. Then the recurrence classes of the unperturbed process are $D^0$ and all singletons $z \in C^0$.
\end{lemma}
The proof is provided in Appendix \ref{sec:app_B}.

Let $D$ be any state in $D^0$ and $z,z' \in C^0$. It can be seen that the resistances for the paths $z \to D$, $D \to z$ and $z \to z'$ in our algorithm are the same as those in \cite[Section 4]{marden}. For instance, the transition from $z \to D$ occurs only when a player explores or the utility is miscalculated. Since we have $p_\epsilon \leq \epsilon^\kappa$, the probability of this event is $O(\epsilon^\kappa)$ and hence the resistance of the transition is $c$. Similarly it can be seen that the resistance for the path $D \to z$ is $(K - \sum_{j \in [K]} \Bar{u}_j)$ and $z \to z'$ is bounded in $[c,2c)$.

The rest of the proof of Theorem 2 follows from \cite[Theorem 3.2]{marden}. Thus, the stochastically stable states of our matching phase algorithm maximize the sum of the utilities of the players. Since we assume a unique optimal action profile, the state with the baseline actions and utilities corresponding to the optimal action profile and all players being content is the stochastically stable state. 

In the following lemma, we bound the probability of the optimal action profile not being played during the exploitation phase of epoch $\ell$ (identified from the matching phase of epochs $\lceil\frac{\ell}{2}\rceil$ to $\ell$), given that event $E^\ell$ does not occur (\textit{i.e.}, the exploration phases of epochs $\lceil\frac{\ell}{2}\rceil$ to $\ell$ were successful).

\begin{lemma}\label{lem4}
In some epoch $\ell$, let 
\[
\mathbf{a}^*=\mathop{\arg\max}\limits_{\mathbf{a} \in \mathcal{A}} \sum_{j=1}^K u_j(\mathbf{a})
\]
and let $\mathbf{a'} = [a'_1,...,a'_K]$ where 
\[
a'_j = \mathop{\arg\max}\limits_{m \in [M]}  \sum_{i = \lceil\frac{\ell}{2}\rceil}^{\ell} W^i{(j,m)}
\]
is the action profile played in the exploitation phase of epoch $\ell$ by player $j$. 

Assume that for all players $j \in [K]$, for all arms $m \in [M]$ and all $n \in [N]$, the estimated mean rewards obtained at the end of the exploration phase for  epochs $\lceil\frac{\ell}{2}\rceil \leq i \leq \ell$ satisfy $|\hat{\mu}_j(m,n) - \mu_j(m,n)| \leq \Delta$. Then
\[
P\{\mathbf{a^*} \neq \mathbf{a'}\} \leq \left(C_0 \exp{(- C_\rho \ell^{\delta/2})}\right)^\ell
\]
for some $C_0, C_\rho > 0$.
\end{lemma} 

The proof relies on using Chernoff-Hoeffding bounds for Markov chains (\cite[Theorem 3]{chung}) and is provided in Appendix \ref{sec:app_B}.

\section{Simulation Results}\label{experiments}

In this section, we present some illustrative simulation results. We consider two cases, one with $K = M$ and the other with $K > M$. In both cases, we set $N = 2$. The mean rewards $\mu_j(m,1)$ for player $j$, arm $m$ are generated uniformly at random from $[0.2,0.95]$. The mean rewards $\mu_j(m,2)$ are generated as $\mu_j(m,2) = 0.5\mu_j(m,1) + u$, where $u$ is a uniform random variable in $[-0.05,0.05]$. The rewards are generated from a uniform distribution in the range [-0.18, 0.18] around the mean. We set $\delta = 0$ (see Appendix \ref{app:exp} for details). 

\begin{figure}
\begin{center}
\centerline{\includegraphics[scale = 0.42]{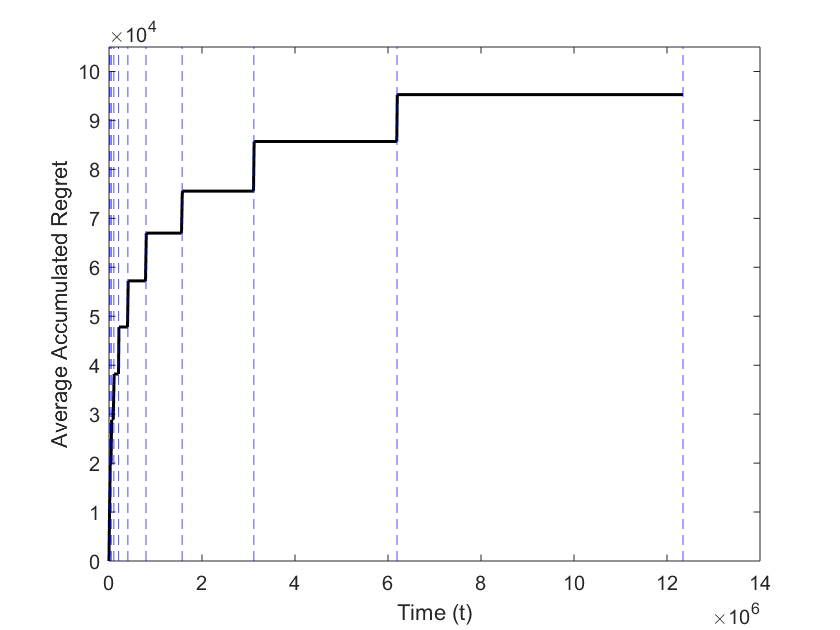}}
\caption{Average accumulated regret as a function of time}
\label{one}
\end{center}
\vskip -0.3in
\end{figure}

Due to numerical considerations, we modify the probabilities in \eqref{state} as $\epsilon^{u_j^{\mathrm{max}} - \hat{u}_j}$ and $1-\epsilon^{u_j^{\mathrm{max}} - \hat{u}_j}$ instead of $\epsilon^{1- \hat{u}_j}$ and $1-\epsilon^{1-\hat{u}_j}$, where $u_j^{\mathrm{max}}$ is the maximum utility that can be received by the player. Further explanation of this is given in Appendix \ref{app:exp}.

To the best of our knowledge, the setting we have considered that allows for heterogeneous reward distributions and non-zero rewards on collisions has not been studied prior to this work. Therefore, it is not possible to compare our algorithms against existing algorithms. A naive uniform pull algorithm, where every player uniformly chooses an arm at each time instant would give linear regret, which would be worse than our proposed algorithm.

\begin{figure}
\begin{center}
\centerline{\includegraphics[scale = 0.42]{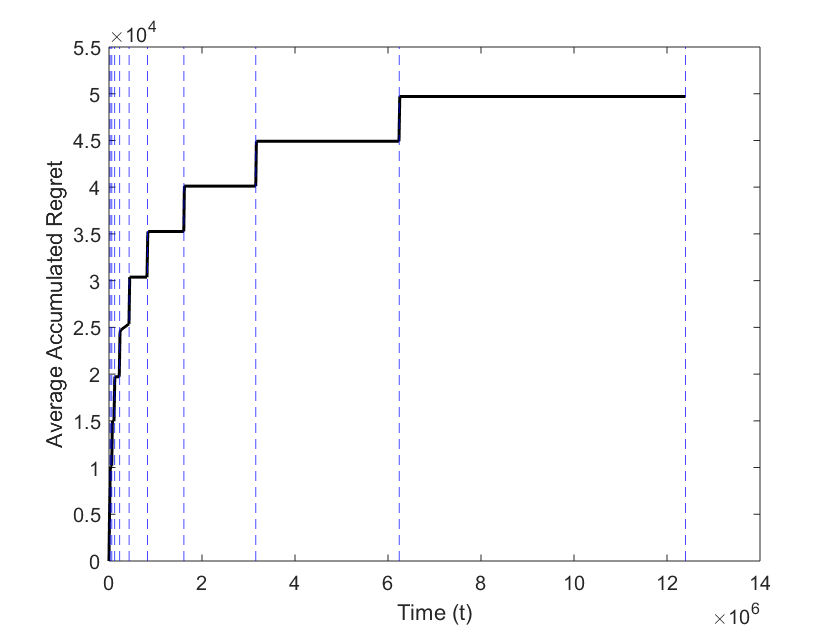}}
\caption{Average accumulated regret as a function of time}
\label{two}
\end{center}
\vskip -0.3in
\end{figure}

For the case with $K = M$, we consider a system with $K = 6$ players and $M = 6$ arms. The optimal action profile is $\mathbf{a^*} = [2\; 1 \;1 \; 6 \; 4 \; 5]$. The considered system has $\Delta = 0.05$. We set $T_0 = 800$ , $c_2 = c_3 = 6 \times 10^3$. The value of $\epsilon$ is set to $10^{-5}$. Since it is possible for each player to play a distinct arm and receive non-zero rewards, $u_j^{\mathrm{max}} = \mathop{\max}_{m \in [M]} \hat{\mu}_j(m,1) $ for all $j \in [K]$. The algorithm was run for $10$ epochs and the experiment was repeated for $100$ iterations and the accumulated regret averaged over the iterations.

For the case with $K > M$, we consider a system with $K = 6$ players and $M = 3$ arms. The optimal action profile is $\mathbf{a^*} = [1\; 2 \;3 \; 1 \; 2 \; 3]$. The considered system has $\Delta = 0.0305$. We set $T_0 = 2 \times 10^3$, $c_2 = c_3 =  10^4$. The value of $\epsilon$ is set to $10^{-5}$. Since it is required for 2 players to choose each arm for all the players to receive non-zero rewards, $u_j^{\mathrm{max}} = \mathop{\max}_{m \in [M]} \hat{\mu}_j(m,2) $. The algorithm was run for $10$ epochs and the experiment was repeated for $100$ iterations and the accumulated regret averaged over the iterations. 

From Figure \ref{one} and Figure \ref{two}, we see that the average accumulated regret grows sub-linearly with time. We can also observe that the average regret  incurred during the exploitation phase of each epoch is small as the matching phase converges to the optimal action profile with high probability.
\section{Conclusion}
%
The multi-player multi-armed bandit approach has been used extensively in recent literature to model uncoordinated dynamic spectrum access systems. The aim of this work was to pose the more realistic and challenging problem setup of multi-player multi-armed bandits with heterogeneous reward distributions and non-zero rewards on collisions, and provide a completely decentralized algorithm achieving near order-optimal regret. We considered a setting where the players cannot communicate with each other and can observe only their own actions and rewards. While settings with non-zero rewards on collisions and heterogeneous reward distributions of arms across players have been considered separately in prior work, a model allowing for both has not been studied previously to the best of our knowledge. With this setup, we presented a policy that achieves near order optimal expected regret of order $O(\log^{1 + \delta} T)$ for some $0 < \delta < 1$ over a time-horizon of duration $T$. Our results have applications to non-orthogonal multiple access systems (NOMA) (see e.g. \cite{marie} for some recent results in this area). A common assumption in most works studying the multi-player multi-armed bandit setup, including ours, is time synchronization. A possible direction of future study in this area is to examine this problem without the assumption that players are time synchronized.


%

\appendices
\section{Exploration Phase: Proof of Lemma \ref{lem1}}\label{app:exploration}

We want to bound the probability of the event $E_{j,m,n}^\ell$ that after the exploration phase of epoch $\ell$, there exists at least epoch $\lceil \frac{\ell}{2} \rceil \leq i \leq \ell $ such that the estimated mean reward after the exploration phase of epoch $i$ $\hat{\mu}_j(m,n)$ satisfies $|\mu_j(m,n) - \hat{\mu}_j(m,n)| \geq \Delta$.

For a fixed player $j$, arm $m$ and number of players on the arm $n \leq N$, the number of total number of reward samples obtained after the exploration phase of epoch $i$ from the corresponding reward distribution with mean reward $\mu_j(m,n)$ is $T_0 \sum_{j = 1}^i j^\delta \geq T_0 (\frac{i}{2})^{1 + \delta}$. After the exploration phase of epoch $i$, the probability that the estimated mean reward $\hat{\mu}_j(m,n)$ deviates from $\mu_j(m,n)$ by more than $\Delta$ is:

\begin{align*}
    &P\left\{\text{After epoch $i$ :} \; |\hat{\mu}_j(m,n) - \mu_j(m,n)| \geq \Delta \right\} \\
    &\leq e^{-2T_0\Delta^2 \sum_{j = 1}^i j^\delta} \\
    &\leq e^{-2T_0\Delta^2 (\frac{i}{2})^{1 + \delta}}
\end{align*}

It follows that
\begin{equation}
    \begin{split}
        &P\left(E_{j,m,n}^\ell \right) \\
    &= P \left(\bigcup_{i = \lceil \frac{\ell}{2} \rceil}^{\ell} \left\{\text{After epoch $i$ :} \; |\hat{\mu}_j(m,n) - \mu_j(m,n)| \geq \Delta \right\} \right) \\
    &\leq \sum_{i = \lceil \frac{\ell}{2} \rceil}^{\ell} P\left\{\text{After epoch $i$ :} \; |\hat{\mu}_j(m,n) - \mu_j(m,n)| \geq \Delta \right\} \\
    &\leq \sum_{i = \lceil \frac{\ell}{2} \rceil}^{\ell} e^{-2T_0\Delta^2 (\frac{i}{2})^{1 + \delta}} \\
    &\leq \sum_{i = \lceil \frac{\ell}{2} \rceil}^{\ell} e^{-T_0\Delta^2 (\frac{\ell}{4})^{\delta}i} \leq \frac{e^{- \frac{T_0 \Delta^2}{2} (\frac{\ell}{4})^{\delta} \ell}}{1 - e^{-T_0 \Delta^2 (\frac{\ell}{4})^{\delta}}}.
    \end{split}
\end{equation}

\section{Matching Phase} \label{sec:app_B}

By definition of the stochastically stable state, this state is played for a majority of time eventually. The matching phase algorithm presented in our paper differs from the dynamics presented in \cite{marden} in two aspects. Despite these two differences, we show that the proof technique  of \cite[Theorem 3.2]{marden} can be adapted to our dynamics as well.

\subsection{Proof of Lemma~\ref{recurring}}
 In the unperturbed process, if all the players are discontent, they remain discontent with probability $1$. Thus we have that $D^0$ represents a single recurrence class. In each state $z \in C^0$, each player chooses their baseline action, and since the utilities received would be the same as the baseline utilities, each player stays content with probability $1$. Thus we have that $D^0$ and the singletons $z \in C^0$ are recurrence classes. 
 
To see that the above are the only recurrence classes, look at any state that has atleast one discontent player and atleast one content player. We have that the baseline actions and utilities of the content players are aligned (since this is the unperturbed process). Consider one of the discontent players. This player chooses an action at random and there is a positive probability (bounded away from 0) of choosing the action of a content player. This would cause the utility of the content player to become misaligned with his baseline utility, thus leading to that player becoming discontent. This continues until all players become discontent. Thus any such state cannot be a recurrent state. 

Now consider a state where all agents are content, but there is at least one player $j$ whose benchmark action and utility are not aligned. For the unperturbed process, in the following step, the same action profile would be played but this would cause player $j$ to become discontent and it follows from the previous argument that this leads to all players becoming discontent. 

Thus we have that $D^0$ and all singletons in $C^0$ are the only recurrent states of the unperturbed process.

\subsection{Proof of Lemma \ref{lem4}}

Since it is given that the exploration phases for all epochs $\lceil \frac{\ell}{2} \rceil$ to $\ell$ are successful, the efficient action profiles maximizing the sum of utilities in the matching phase for all epochs $\lceil \frac{\ell}{2} \rceil \leq i \leq \ell$ are the same, and are equal to the optimal action profile maximizing the sum of expected mean rewards:

$$\mathop{\arg\max}_{\mathbf{a} \in \mathcal{A}} \sum_{i=1}^{K} \mu_j(a_j,k(a_j)) = \mathop{\arg\max}_{\mathbf{a} \in \mathcal{A}} \sum_{i=1}^{K} u_j(\mathbf{a}).$$

Let $\mathbf{\Bar{u}^\ell}$ denote the utilities of the players for the optimal action profile $\mathbf{a^*}$ during epoch $\ell$. The optimal state of the Markov chain is then ${z^*}^\ell = [\mathbf{a^*},\mathbf{\Bar{u}^\ell},C^K ]$ during epoch $\ell$. Note that optimal state differs only in the baseline utilities for epochs $\lceil\frac{\ell}{2}\rceil$ to $\ell$.  In order to bound the probability of the event $\{\mathbf{a^*} \neq \mathbf{a'}\}$, we use the Chernoff Hoeffding bounds for Markov chains from \cite{chung}, which is also used in \cite{got} and \cite{bistritz2020game}. When the Markov chain is in state $z$, the estimated/observed state is $\hat{z}$, corresponding to the estimated states of all the players. The function $f(z)$ considered here in order to use the bound from \cite[Theorem 3]{chung} for epoch $\ell$ is:

\begin{equation}
    f(z) = \identityf{\hat{z} = {z^*}^{\ell}}, 
\end{equation}

\textit{i.e.} the estimated state is the optimal state. Recall that $\tau_i = \sqrt{c_2} i^{\delta/2}$ (replacing $\ell$ by $i$ in $\tau_\ell$). It follows that 

\begin{align}
    P\left\{\mathbf{a'} \neq \mathbf{a^*}\right\} \leq P\left\{\sum_{i = \lceil \frac{\ell}{2}\rceil}^\ell \sum_{h = 1}^{\tau_i} f(z_h) \leq \frac{1}{2}\sum_{i = \lceil \frac{\ell}{2}\rceil}^\ell \tau_i \right\}.
\end{align}

Define 

\begin{equation}
    X_i = \sum_{h = 1}^{\tau_i} \identityf{\hat{z}_h = {z^*}^{i}}
\end{equation}

\begin{equation}
    L = \frac{1}{2}\sum_{i = \lceil \frac{\ell}{2}\rceil}^\ell \tau_i.
\end{equation}

Using the Chernoff bound, for some $s > 0$, it follows that

\begin{align}
    P\left\{\sum_{i = \lceil \frac{\ell}{2}\rceil}^\ell X_i \leq \frac{1}{2}\sum_{i = \lceil \frac{\ell}{2}\rceil}^\ell \tau_i \right\} &= P\left\{e^{-s\sum_{i = \lceil \frac{\ell}{2}\rceil}^\ell X_i} \geq e^{-sL} \right\} \\
    &\leq e^{sL}\Pi_{i = \lceil \frac{\ell}{2}\rceil}^\ell \expect{e^{-sX_i}} \label{chernoff}.
\end{align}

In order to use the bound from \cite[Theorem 3]{chung}, we need to calculate $\mu^\ell = \expect{f(z)} = P(\hat{z} = {z^*}^\ell)$. Observe that

\begin{align}
    \expect{f(z)} &= P\{\hat{z} = {z^*}^\ell\} \\
    &\geq P\{{z} = {z^*}^\ell , \hat{z} = z\} \\
    &\geq P\{{z} = {z^*}^\ell\}(1 - K\epsilon^\kappa)
\end{align}

where the last step follows from Lemma \ref{lem2}. 

Define 

\begin{align}
    \pi_z &= \min_{\lceil \frac{\ell}{2}\rceil \leq i \leq \ell} P\{{z} = {z^*}^i\} \\
    \mu &= \min_{\lceil \frac{\ell}{2}\rceil \leq i \leq \ell} \mu^i.
\end{align}

From the definition of a stochastically stable state, we can choose an $\epsilon$ small enough such that 

\begin{equation}
    \mu \geq \pi_z (1 - K \epsilon^\kappa) > \frac{1}{2(1-\eta)}
\end{equation}

for some $0 < \eta < 1/2$. 

We can now use the bound from \cite[Theorem 3]{chung} for epoch $\lceil \frac{\ell}{2}\rceil \leq i \leq \ell$ to get

\begin{align}
    P\left\{X_i \leq  \frac{\tau_i}{2} \right\} &\leq P\left\{\sum_{h = 1}^{\tau_i} \identityf{\hat{z}_h = {z^*}^{i}} \leq (1 - \eta)\mu^i \tau_i \right\} \\
    &\leq c_0 \|\phi_i\|_\pi \exp\left(- \frac{\eta^2 \mu^i \tau_i}{72T}\right)\\
    &\leq c_0 \|\phi_i\|_\pi \exp\left(- \frac{\eta^2 \mu \tau_i}{72T}\right)
\end{align}

where $c_0 > 0$, $\phi_i$ is the initial distribution of the Markov chain in the $i$-th epoch and $T$ is the 1/8-th mixing time of the Markov chain. 
Using $s = \frac{\eta^2}{(1 - \eta)72T}$ it follows that,  

\begin{align}
    \expect{e^{-sX_i}} \leq (1 + c_0 \|\phi_i\|_\pi)  \exp\left(- \frac{\eta^2 \mu \tau_i}{72T} \right).
\end{align}

Using the above in \eqref{chernoff}, 

\begin{align}
     &P\left\{\sum_{i = \lceil \frac{\ell}{2}\rceil}^\ell X_i \leq \frac{1}{2}\sum_{i = \lceil \frac{\ell}{2}\rceil}^\ell \tau_i \right\} \\
     &\leq \Pi_{i = \lceil \frac{\ell}{2}\rceil}^\ell (1 + c_0 \|\phi_i\|_\pi) e^{\frac{\eta^2L}{(1 - \eta)72T}} e^{- \frac{\eta^2 \mu \tau_i}{72T}} \\
     &\leq C_0^\ell \exp \left({-\frac{\eta^2(\mu - \frac{1}{2(1 - \eta)})2L}{72T}}\right) \\ 
     &\leq C_0^\ell \exp \left({-\frac{{\eta^2 \sqrt{c_2}} {2^{-(1 + \delta/2)}}(\mu - \frac{1}{2(1 - \eta)}) \ell^{1 + \delta/2}}{72T}}\right) \\
     &\leq \left(C_0 \exp \left( -C_\rho \ell^{\delta/2} \right)\right)^\ell
\end{align}

where $C_0 = \max_{\lceil \frac{\ell}{2}\rceil \leq i \leq \ell} (1 + c_0 \|\phi_i\|_\pi)$, and 
\begin{equation}\label{crho}
 C_\rho = \frac{{\eta^2 \sqrt{c_2}} {2^{-(1 + \delta/2)}}(\mu - \frac{1}{2(1 - \eta)}) }{72T} > 0.   
\end{equation}

\section{Details on Simulations}\label{app:exp}

Due to numerical considerations, we modify the state update step in the matching phase of the algorithm, which is defined in \eqref{state}.

We modify the probabilities in \eqref{state} as $\epsilon^{u_j^{\mathrm{max}} - \hat{u}_j}$ and $1-\epsilon^{u_j^{\mathrm{max}} - \hat{u}_j}$, instead of $\epsilon^{1- \hat{u}_j}$ and $1-\epsilon^{1-\hat{u}_j}$ respectively. Here $u_j^{\mathrm{max}}$ denotes the maximum utility that can be received by player $j$, without reducing the utility of any other player to zero. For example, when there are $K = 6$ players and $M = 6$ arms, each player can occupy a separate arm and all the players could receive non-zero utilities. Thus the value of $u_j^{\mathrm{max}}$ of player $j$ would be $u_j^{\mathrm{max}} = \mathop{\max}_{m \in [M]} \hat{\mu}_j(m,1)$. Consider the case when there are $K = 6$ players, $M = 3$ arms and $N = 2$. In this case, it is not possible for any player $j$ to occupy an arm by himself without reducing the utilities of some other players to zero. In this case, $u_j^{\mathrm{max}} = \mathop{\max}_{m \in [M]} \hat{\mu}_j(m,2)$. 

The algorithm presented here is motivated for small systems. As the system size increases, $\Delta$ becomes smaller and hence it becomes difficult to implement the proposed algorithm. Larger systems could be accommodated by dividing into subsystems and applying the protocol separately on each subsystem.  

Note that $\delta > 0$ was required in the proof of Theorem 1 to bound the regret incurred during the exploration and matching phases. However, in practice we observe that setting $T_0$ and $c_2$ large enough guarantees that equations \eqref{exp1} and \eqref{match1} are satisfied even with $\delta = 0$.

The order of magnitude of $\epsilon$ is important for the performance of the algorithm. However, the exact value of $\epsilon$ does not play much of a role. For example, in our simulations, we use $\epsilon = 10^{-5}$, and similar results are achieved for $\epsilon = 10^{-4}$ as well. The value of $\epsilon$ needs to be chosen small enough such that the state corresponding to the optimal action is visited a certain number of times during the matching phase. Empirically, we observe that as the value of $\Delta$ decreases, we need a smaller value of $\epsilon$ for convergence. We find that the value of $\epsilon$ depends on the optimality gap $J_1 - J_2$ and the size of the system. While \cite{got} provides a upper bound on the value of $\epsilon$, this bound depends on the analysis of the resistance tree structure of the perturbed Markov chain and needs to be computed beforehand for each setting. Note that such an analysis involving the resistance tree structures of our algorithm would be tedious and the resulting bound would be expensive to calculate. Since we only need $\epsilon$ to be of the right order of magnitude, we could tune this parameter using simulations.

\section*{Acknowledgment}
The authors of this paper would like to thank Aditya Deshmukh for valuable discussions.


\ifCLASSOPTIONcaptionsoff
  \newpage
\fi



\bibliographystyle{IEEEtran}
\bibliography{ref.bib}
%

%






\begin{IEEEbiographynophoto}{Akshayaa Magesh}(S '20)
 received her B.~Tech. degree in Electrical Engineering from the Indian Institute of Technology Madras, Chennai, India in 2018, and her Masters degree from the Department of Electrical and Computer Engineering at the University of Illinois at Urbana-Champaign in 2020. She is now a Ph.D. candidate at the Department of Electrical and Computer Engineering at the University of Illinois at Urbana-Champaign. Her research interests span the areas of statistical inference, theoretical and algorithmic aspects of robust machine learning and reinforcement learning.
 \end{IEEEbiographynophoto}

\begin{IEEEbiographynophoto}{Venugopal V. Veeravalli (M'92, SM'98, F'06)}  received the B.Tech. degree (Silver Medal Honors) from the Indian Institute of Technology, Bombay, in 1985, the M.S. degree from Carnegie Mellon University, Pittsburgh, PA, in 1987, and the Ph.D. degree from the University of Illinois at Urbana-Champaign, in 1992, all in electrical engineering. He joined the University of Illinois at Urbana-Champaign in 2000, where he is currently the Henry Magnuski Professor in the Department of Electrical and Computer Engineering, and where he is also affiliated with the Department of Statistics, and the Coordinated Science Laboratory. Prior to joining the University of Illinois, he was on faculty of the ECE Department at Cornell University. He served as a Program Director for communications research at the U.S. National Science Foundation from 2003 to 2005. His research interests span the theoretical areas of statistical inference, machine learning, and information theory, with applications to data science, wireless communications, and sensor networks. He was a Distinguished Lecturer for the IEEE Signal Processing Society during 2010--2011. He has been on the Board of Governors of the IEEE Information Theory Society. He has been an Associate Editor for Detection and Estimation for the IEEE Transactions  on Information Theory and for the IEEE Transactions on Wireless Communications. Among the awards he has received for research and teaching are the IEEE Browder J. Thompson Best Paper Award, the Presidential Early Career Award for Scientists and Engineers (PECASE), and the Wald Prize in Sequential Analysis.
\end{IEEEbiographynophoto}

\vfill
\end{document}